\documentclass{article}

\usepackage{arxiv}

\usepackage[utf8]{inputenc} 
\usepackage[T1]{fontenc}    
\usepackage{hyperref}       
\usepackage{url}            
\usepackage{booktabs}       
\usepackage{amsfonts}       
\usepackage{nicefrac}       
\usepackage{microtype}      
\usepackage{lipsum}

\usepackage{cite}
\usepackage{amsmath,amssymb,amsfonts}
\usepackage{graphicx}
\usepackage{algorithm}
\usepackage[noend]{algpseudocode}
\usepackage{hyperref}
\usepackage{textcomp}
\usepackage{threeparttablex} 

\usepackage{mathtools}
\usepackage{caption}
\usepackage{float}
\usepackage{stmaryrd}
\usepackage{epstopdf}
\usepackage{multirow}
\usepackage{pifont}
\usepackage{caption}
\usepackage{subcaption}
\usepackage{xcolor}

\newcommand{\norm}[1]{\left\lVert#1\right\rVert}

\newcommand{\R}{{\mathbb{R}}}

\newcommand{\B}{{\mathcal B}}

\newcommand{\N}{{\mathbb{N}}}

\newcommand{\X}{{\mathbf{X}}}

\newcommand{\T}{{\mathbf{T}}}
\newcommand{\So}{{\mathbf{S}}}
\newcommand{\Obs}{{\mathcal{U}}}

\newcommand{\U}{{\mathbf{U}}}
\newcommand{\Y}{{\mathbf{Y}}}

\newcommand{\cen}{\sigma}
\newcommand{\rad}{\rho}
\newcommand{\ex}{\mathsf{e}}

\newtheorem{theorem}{Theorem}[section]

\newtheorem{assumption}{Assumption}

\newtheorem{definition}[theorem]{Definition}
\newtheorem{lemma}[theorem]{Lemma}
\newtheorem{remark}[theorem]{Remark}
\newtheorem{problem}[theorem]{Problem}
\newenvironment{proof}{\par\noindent\textbf{Proof.} }{\hfill$\blacksquare$\par}

\title{Real-Time Spatiotemporal Tubes for Dynamic Unsafe Sets
\thanks{ This work was supported in part by the ARTPARK. The work of Ratnangshu Das was supported by the Prime Minister’s Research Fellowship from the Ministry of Education, Government of India.}
}

\author{
 Ratnangshu Das \\
  Robert Bosch Centre for Cyber-Physical Systems\\
  IISc, Bengaluru, India\\
  \texttt{ratnangshud@iisc.ac.in} \\
   \And
 Siddhartha Upadhyay \\
  Robert Bosch Centre for Cyber-Physical Systems\\
  IISc, Bengaluru, India\\
  \texttt{siddharthau@iisc.ac.in} \\
   \And
 Pushpak Jagtap \\
  Robert Bosch Centre for Cyber-Physical Systems\\
  IISc, Bengaluru, India\\
  \texttt{pushpak@iisc.ac.in} \\
}

\begin{document}
\maketitle

\begin{abstract}
    This paper presents a real-time control framework for nonlinear pure-feedback systems with unknown dynamics to satisfy reach-avoid-stay tasks within a prescribed time in dynamic environments. To achieve this, we introduce a real-time spatiotemporal tube (STT) framework. An STT is defined as a time-varying ball in the state space whose center and radius adapt online using only real-time sensory input. A closed-form, approximation-free control law is then derived to constrain the system output within the STT, ensuring safety and task satisfaction. We provide formal guarantees for obstacle avoidance and on-time task completion. The effectiveness and scalability of the framework are demonstrated through simulations and hardware experiments on a mobile robot and an aerial vehicle, navigating in cluttered dynamic environments.
\end{abstract}

\section{Introduction}
Autonomous systems must safely and efficiently execute complex task specifications with strict time constraints. Tasks such as reaching a target, avoiding obstacles, and remaining within safe bounds, collectively known as temporal reach-avoid-stay (T-RAS) specifications, are essential in applications like autonomous driving, drone delivery, and robotic exploration. They also serve as building blocks in high-level task-planning frameworks using temporal logic \cite{jagtap2020formal, kloetzer2008fully}. It becomes even more challenging when the system dynamics are unknown and the environment is dynamic, with obstacles that change unpredictably. In such settings, there is a growing need for real-time control strategies that ensure safety and prescribed-time task completion.

Several approaches have been proposed for navigation in dynamic environments. Classical techniques such as Artificial Potential Fields (APF) \cite{APF_global} uses virtual electrostatic-like forces to guide the agent toward the goal while avoiding obstacles, but struggles with local minima in cluttered environments. 
Similarly, the Dynamic Window Approach (DWA) \cite{fox2002dynamic} offers real-time collision avoidance via velocity sampling but lacks formal safety guarantees and scales poorly in dense, dynamic environments.
Other path planning algorithms \cite{Path_Planning_review} include graph-based methods, like A* and Dijkstra, and sampling-based methods \cite{sampling_review}, like PRM and RRT, that can generate feasible paths. However, these solutions typically require separate tracking controllers, lack formal guarantees, and fail to enforce prescribed-time requirements.

Optimization-based approaches, such as Model Predictive Control (MPC) \cite{MPC} and Control Barrier Functions (CBFs) \cite{CBF}, offer real-time safety enforcement under known dynamics. Significant efforts have been made in works such as \cite{CBF-APF,CBF-RRT1} to integrate CBFs with APF or sampling-based planners to improve safety in dynamic environments. In \cite{tayal2024collision}, the authors proposed a Collision Cone Control Barrier Function (C3BF) framework to enable safe trajectory tracking in dynamic environments. The idea of combining CBF with control Lyapunov function (CLF) has also been used in \cite{CLF_CBF_motion} to avoid obstacles in real time. While effective, these methods typically require accurate system models and online optimization, making them less suitable for high-dimensional or uncertain systems.

Funnel-based methods have emerged as an elegant way to design feedback controllers that ensure trajectories remain within guaranteed bounds \cite{Funnel_1}. These methods have been effective for reachability and tracking \cite{Funnel_2}, but tasks such as avoiding unsafe regions remain challenging \cite{lindemann2021funnel}. Some studies \cite{RB} integrate path planning with funnel-based tracking by planning paths within an extended free space, created by obstacle expansion \cite{KDF_1}. However, they only address static obstacles, which limits their use in dynamic environments.

The Spatiotemporal Tubes (STT) framework \cite{das2024prescribed} offers a time-varying, goal-directed tube in state space that ensures safety and prescribed-time task completion. It has also been extended to handle more complex tasks \cite{STT_omega, STT_STL} and multi-agent systems \cite{STT_Multi}. However, the approach in \cite{das2024prescribed} constructs tubes using circumvent functions, which can introduce abrupt changes in tube geometry and are restricted to static unsafe sets. {In contrast, the sampling-based optimization approach in \cite{das2024spatiotemporal} can handle time-varying unsafe sets, but it requires complete knowledge of obstacle trajectories in advance and involves computationally expensive offline synthesis. These limitations make existing methods unsuitable for real-time use in environments where obstacles may change dynamically.}

In this work, we introduce a {real-time STT} framework to address these challenges. Specifically, we propose a spherical STT, modeled as a time-varying ball in the state space whose center and radius {adapt online based on real-time sensory input.} By constraining the unknown system to remain within this evolving tube, we ensure that it avoids the dynamic unsafe set and reaches the target within a prescribed time.
{The main contributions of the work can be summarized as:
\begin{itemize}
    \item A real-time STT formulation that ensures safe prescribed-time convergence in environments with dynamic unsafe sets. The approach relies only on current sensor data and does not need any offline computation or prior knowledge of unsafe set trajectories.
    \item A computationally lightweight and model-free controller that keeps the system trajectory within the STT under unknown dynamics and disturbances, providing formal safety and timing guarantees suitable for real-time operation.
\end{itemize}
}
We validate the proposed framework through extensive simulation and hardware experiments on a 2D omnidirectional mobile robot and a 3D Quadrotor. {We also present detailed comparisons with state-of-the-art algorithms, demonstrating better computation efficiency and success rate.}

\section{Preliminaries and Problem Formulation}
\label{sec:prelim}
\subsection{Notation}
For $a,b\in\N$ with $a\leq b$, we denote the closed interval in $\N$ as $[a;b] := \{a, a+1, \ldots, b\}$. A ball centered at $\cen \in \mathbb{R}^n$ with radius $\rad \in \mathbb{R}^+$ is defined as $\mathcal{B}(\cen, \rad) := \{ x \in \mathbb{R}^n \mid \|x - \cen\| \leq \rad \}$. We use $x\circ y$ to represent the element-wise multiplication where $x,y\in \R^n$.  An identity matrix of order $n\in \N$ is represented using $I_n$. All other notation in this paper follows standard mathematical conventions.

\subsection{System Definition}
We consider a class of control-affine, multi-input multi-output (MIMO), nonlinear pure-feedback systems. {This structure is common in many practical systems such as torque-controlled manipulators, acceleration-controlled quadrotors, and magnetic levitation systems \cite{Funnel_1, das2024spatiotemporal, PPCfeedback}.}
It is described by the following dynamics:
\begin{align} \label{eqn:sysdyn}
    &\dot{x}_k(t) = f_k(\overline{x}_k(t)) + g_k(\overline{x}_k(t))x_{k+1}(t) + w_k(t), \notag \\
    &\hspace{5cm} \text{for all } k\in [1;N-1], \notag\\
    &\dot{x}_{N}(t) = f_N(\overline{x}_N(t)) + g_N(\overline{x}_N(t))u(\overline{x}_N, t) + w_N(t), \\
    &y(t) = x_1(t), \nonumber
\end{align}
where for each $t\in\R^+_0$ and $k\in[1;N]$,
\begin{itemize}
    \item $x_k(t) = [x_{k,1}(t), \ldots, x_{k,n}(t)]^\top \in {\X}_k \subset \mathbb{R}^{n}$ is the state,
    \item $\overline{x}_k(t) := [x_1^\top(t),...,x_k^\top(t)]^\top \in \overline{\X}_k = \prod_{j=1}^k \X_j \subset \mathbb{R}^{nk} $,
    \item $u(\overline{x}_N, t) \in \mathbb{R}^n$ is control input vector,
    \item $w_k(t) \in \mathbf{W} \subset \R^n$ is unknown bounded disturbance, and
    \item $y(t) = [x_{1,1}(t), \ldots, x_{1,n}(t)]\in \Y=\X_1$ is the output.
\end{itemize}

The functions $f_k: \overline{\X}_k \rightarrow \mathbb{R}^n$ and $g_k: \overline{\X}_k \rightarrow \mathbb{R}^{n \times n}$ satisfy the following assumptions:
\begin{assumption}\label{assum:lip}
    For all $k \in [1;N]$, the functions $f_k$ and $g_k$ are unknown but locally Lipschitz continuous.
\end{assumption}
\begin{assumption}[\cite{PPC1,PPC0}] \label{assum:pd}
    For all $\overline{x}_k \in \overline{\X}_k$, the symmetric part of $g_k$, defined as $g_{k,s}(\overline{x}_k) := \frac{g_k(\overline{x}_k)+g_k(\overline{x}_k)^\top}{2}$ is uniformly sign definite with known sign. Without loss of generality, we assume $g_{k,s}(\overline{x}_k)$ is positive definite, that is, there exists a constant $\underline{g_k}\in\mathbb R^+, \forall k \in [1;N]$ such that
    $0 < \underline{g_k} \leq \lambda_{\min} (g_{k,s}(\overline{x}_k)), \forall \ \overline{x}_k \in \overline{\X}_k$,
    where $\lambda_{\min}(\cdot)$ denotes the smallest eigenvalue of a matrix.
\end{assumption}
This assumption ensures that in \eqref{eqn:sysdyn} global controllability is guaranteed, i.e., $g_{k,s}(\overline{x}_k) \neq 0,$ for all $\overline{x}_k \in \overline{\X}_k$.

\subsection{Problem Formulation}
Let the output $y(t)$ of system \eqref{eqn:sysdyn} be subject to a \textit{temporal reach-avoid-stay (T-RAS) specification} \cite{das2024spatiotemporal} defined next.

\begin{definition}[Temporal Reach-Avoid-Stay (T-RAS) task]\label{def:tras}
Given the output-space $\mathbf{Y}=\X_{1}$, a prescribed time $t_c \in \R^+$, a time-varying unsafe set $\U: \R_0^+ \rightarrow \mathcal{P}(\Y)$, an initial set $\So \subset \mathbf{Y} \setminus \U(0)$, and a target set $\T \subset \mathbf{Y} \setminus \U(t_c)$, we say the output $y(t)$ satisfies the T-RAS task if: 
$$y(0) \in \So, \quad y(t_c) \in \T, \text{ and } y(t) \in \mathbf{Y} \setminus \U(t), \forall t \in [0, t_c].$$
\end{definition}

\begin{remark}
    The unsafe set $\U(t)$ is defined as a union of $n_o$ time-varying balls, $\Obs^{(j)}(t)$ for $j \in [1;n_o]$, each centered around a dynamic obstacle: 
    $$\U(t) = \bigcup_{j=1}^{n_o} \Obs^{(j)}(t), \text{ with } \Obs^{(j)}(t):= \mathcal{B}(o^{(j)}(t), \rad_o^{(j)}(t)),$$ 
    where $o^{(j)}(t) \in \R^n$ and $\rad_o^{(j)}(t) \in \R_0^+$ denote the time-varying center and radius, capturing the region surrounding the $j^{th}$ dynamic obstacle. Since these regions are defined independently, it allows for modeling multiple, disjoint, and independently evolving unsafe regions.
\end{remark}

We now state the control problem addressed in this work.
\begin{problem}[Real-time T-RAS Control]\label{prob1}
Given the system \eqref{eqn:sysdyn} under Assumptions \ref{assum:lip}-\ref{assum:pd}, and a T-RAS task as defined in Definition \ref{def:tras}, synthesize a \textit{real-time}, \textit{approximation-free}, and \textit{closed-form} control law $u(\overline{x}_N, t)$ that guarantees the output trajectory $y(t)$ satisfies the T-RAS specification.
\end{problem}

To solve Problem~\ref{prob1}, we utilize the STT framework, which defines a time-varying region in the output space that remains safe and goal-directed throughout the time horizon.

\begin{definition}\label{def:stt}
    Given a T-RAS task in Definition \ref{def:tras}, a spatiotemporal tube (STT), $\Gamma(t) = \mathcal{B} (\cen(t), \rad(t))$, is characterized by a time-varying center $\cen: \R_0^+ \rightarrow \mathbb{R}^n$ and radius $\rad: \R_0^+ \rightarrow \mathbb{R}^+$, if the following holds
    \begin{subequations}\label{eqn:stt}
    \begin{align}
        &\rad(t) > 0, \forall t \in [0, t_c],\\
        &\Gamma(0) \subseteq \So,\ \Gamma(t_c) \subseteq \T,\\
        &\Gamma(t) \cap \U(t) = \emptyset, \forall t \in [0, t_c].
    \end{align}
    \end{subequations}
\end{definition}
\begin{remark}
If one designs a control law that constrains the output trajectory within the STT, i.e.,
\begin{align} \label{eqn:stt_con}
    y(t) \in \Gamma(t), \forall t \in [0, t_c],
\end{align}
then one can ensure the satisfaction of T-RAS specification.
\end{remark}

{For solving T-RAS tasks, \cite{das2024spatiotemporal} synthesized STTs offline, assuming full knowledge of future obstacle trajectories. This limits applicability to real-world settings where only real-time sensory information is available. In this work, we overcome this by synthesizing and adapting STTs online using only current environment data. This enables safe operation in dynamic and partially unknown environments.} In the next section, we describe the real-time STT synthesis approach.

\section{Designing Spatiotemporal Tube (STT)} \label{sec:tube}
We begin by selecting the points $s=[s_1,...,s_n]\in \text{int}(\So)$ and $\eta=[\eta_1,...,\eta_n]\in int(\T)$ in the interior of the initial set ($\So$) and the target set ($\T$). Around these points, we define balls of radius $d_S,d_T\in \mathbb{R}^+$, denoted by $\hat\So$ and $\hat\T$, as follows:
\begin{align}
    \hat\So &=\B(s,d_S):=\{x\in \mathbb{R}^n|\norm{x-s}\leq d_S\}\label{eqn:start_state}\\
    \hat\T &=\B(\eta,d_T):=\{x\in \mathbb{R}^n|\norm{x-\eta}\leq d_T\}\label{eqn:targ_state}
\end{align}
such that $\hat \So\subset\So$ and $\hat\T\subset \T$.
As introduced in Definition~\ref{def:stt}, the STT $\Gamma(t) = \mathcal{B}(\cen(t), \rad(t))$ is defined by a time-varying center $\cen: \R_0^+ \rightarrow \R^n$ and radius $\rad: \R_0^+ \rightarrow \R^+$.
Additionally, to ensure a safe approach to the target, we make the following separation assumption:
\begin{assumption}\label{ass_pmin}
    At time $t=t_c$, the unsafe set is separated from the STT center $\cen(t_c)$ by at least a known minimum distance $\rad_{max}\in \R^+$, i.e.,
    $\forall x \in \U (t_c), \|x - \cen(t_c)\| >  \rad_{max}. $
\end{assumption}
{Note that this assumption does not restrict obstacles from being near the target, but requires them to move away as $t$ approaches $t_c$ to ensure a safe, collision-free entry into the target.}

Now, the evolution of the center $\cen(t)$ is governed by the following dynamics:
\begin{align}\label{eqn:cen_dynamic}
    \dot{\cen}(t)&=k_1\frac{t_c}{t_c-t}(\eta-\cen(t)) \notag \\
    &+\sum_{j=1}^{n_o}\Big (k_{2,j}m_j(t) + k_{3,j}v_j(t) \Big )  \theta_j(t), \ \cen(0)=s,
\end{align}
where $k_1>1/t_c$ and $k_{2,j},k_{3,j} \in \R^+$ are arbitrary positive constants. The switching function $\theta_j(t)$ activates the obstacle avoidance condition in the second part of Equation~\eqref{eqn:cen_dynamic} when the STT center $\cen(t)$ approaches the $j$-th dynamic obstacle. 
\begin{equation}\label{eqn:switch}
    \theta_j(t) = 
    \begin{cases}
    &\hspace{-0.7cm}\frac{1}{\norm{\cen(t) - o^{(j)}(t)}-\rad_o^{(j)}(t)}-\frac{1}{\rad_{max}},\\ & \quad\quad\quad\quad\text{if } \| \cen(t) - o^{(j)}(t) \| - \rad_o^{(j)}(t) \leq \rad_{max} \\
    0, & \hspace{1.5cm}\text{otherwise},
    \end{cases}    
\end{equation}
where $\rad_{max} \leq \min(d_S,d_T)$ is the maximum allowable tube radius and can also be interpreted as the sensing radius for detecting nearby obstacles. This implies that the avoidance part is engaged only when the STT center comes within a distance of $\rad_{max}$ of the $j^{th}$ unsafe region. Setting $\rad_{max} \leq \min(d_S,d_T)$ ensures that the tube remains inside $\So $ at $t=0$ and $\T$ at $t=t_c$.

The avoidance term in Equation~in \eqref{eqn:cen_dynamic} is governed by two vectors $ m_j(t), v_j(t) \in \mathbb{R}^n $, where for all $j \in [1;n_o]$
$$ m_j(t) = \frac{\cen(t) - o^{(j)}(t)}{\left( \| \cen(t) - o^{(j)}(t) \| - (\rad_o^{(j)}(t)+\rad_{min})\right)^3}, $$
and $ v_j(t) \in \mathbb{R}^n $ lies in the null space of $m_j(t)$, i.e., $ m_j^\top(t) v_j(t) = 0.$ $\rad_{min} \in \R^+$ is the minimum tube radius.

The tube radius $\rad(t)$ is dynamically adjusted according to the proximity of the tube to obstacles, and evolves as:
\begin{equation}\label{eqn:radius_dynamics}
\dot{\rad}(t)=\frac{\ex^{-\nu d{(t)}}\dot{ d}(t)}{\big ( \ex^{-\nu\rad_{max}}+\ex^{-\nu d(t)}\big)},
\end{equation}
{where $\nu \in \R^+$ is an arbitrary smoothing parameter} and $d(t)$ denotes the smooth minimum distance to obstacles:
\begin{align}\label{eqn:rho}
     d(t) &= -\frac{1}{\nu}\ln \left( \sum_{j=1}^{n_o}\ex^{-\nu \hat d_j(t)} \right), \\
     \hat d_j(t)&=\| \cen(t) - o^{(j)}(t) \| - \rad_o^{(j)}(t), \forall j\in [1;n_o].    \notag
\end{align}

The intuition behind \eqref{eqn:cen_dynamic} is as follows: the first term pulls the STT center $\cen$ toward the target point $\eta$, ensuring convergence in the prescribed time $t_c$. The second term activates only when $\theta_j(t) \neq 0$, i.e., when the STT center is near the $j$-th unsafe set. In that case, the normal vector $m_j(t)$ and the orthogonal vector $v_j(t)$ bend the tube around the unsafe set, ensuring collision avoidance. 
The radius dynamics in \eqref{eqn:radius_dynamics} complements this by smoothly shrinking the tube toward $\rad_{min} \in \R^+$ as it approaches obstacles (i.e., when $ d(t)$ decreases) and expanding toward $\rad_{max}$ when far from unsafe set. 

Given a time-varying center $\cen: \R_0^+ \rightarrow \R^n$ and radius $\rad: \R_0^+ \rightarrow \R^n$, governed by the dynamics in Equations \eqref{eqn:cen_dynamic} and \eqref{eqn:radius_dynamics}, respectively, we define the STT $\Gamma(t) = \mathcal{B} (\cen(t), \rad(t))$ as a closed ball in $\R^n$ centered at $\cen(t)$ with radius $\rad(t)$:
\begin{equation}\label{eqn:stt_ball}
    \Gamma(t) := \{ x\in \R^n \mid \|x-\cen(t)\| \leq \rad(t)\}, \quad \forall t \in [0, t_c] .
\end{equation}
{
\begin{remark}
The parameters $k_1$, $k_{2,j}$, and $k_{3,j}$ control the convergence and avoidance behavior, determining how aggressively the system approaches the target and reacts to nearby obstacles. Larger values make the response faster but can increase control effort. The bounds $\rad_{min}$ and $\rad_{max}$ set the minimum safety margin and the maximum allowed deviation from the nominal path. In practice, $\rad_{min}$ should be at least the robot’s size, while $\rad_{max}$ depends on the desired flexibility and environment density. These parameters can be tuned empirically to balance safety, smoothness, and responsiveness. 
\end{remark}
}

The next theorem guarantees that the designed STT adheres to the following three key conditions for satisfying T-RAS specifications. First, the tube reaches the target set within the prescribed time $t = t_c$. Second, the tube remains entirely outside the unsafe set for all times $t \in [0, t_c]$, ensuring safety. Finally, the radius of the tube remains strictly positive throughout the motion.

\begin{theorem}\label{theorem_tube}
The STT $\Gamma(t)$ in \eqref{eqn:stt_ball} meets the following to ensure satisfaction of the T-RAS specification:
\begin{enumerate}
    \item[(i)] Reaches the target within prescribed time: $\Gamma(t_c) \subseteq \T$.
    \item[(ii)] Avoids the unsafe set: 
    $\Gamma(t) \cap \U(t) = \emptyset$,  $\forall t \in [0, t_c]$.
    \item[(iii)] The radius stays positive: $\rad(t) \in \mathbb{R}^+, $ $\forall t \in [0, t_c]$.
\end{enumerate}
\end{theorem}

\begin{proof}
We now examine each claim individually.

(i) 
By Assumption~\ref{ass_pmin}, the tube is far enough from all unsafe sets at time $t = t_c$; so $\theta^{(j)}(t_c) = 0$, for all $j \in [1;n_o]$.
Thus, at $t = t_c$, the dynamics of the center simplify to:
\begin{equation}\label{updated_cent}
    \dot{\cen}=k_1\frac{t_c}{t_c-t}(\eta-\cen(t)).
\end{equation}
Solving this equation we obtain $\cen(t) = \eta + C(t_c-t)^{k_1 t_c}$, where $C$ is a constant determined by the initial condition. Thus, the STT center approaches $\eta$ as $t$ approaches $t_c$, $\cen(t_c) = \eta$, with convergence rate determined by $k_1 \in \R^+$. 

Further, under Assumption~\ref{ass_pmin}, the radius of the tube expands to its default in the absence of surrounding unsafe sets at $t=t_c$. Therefore, $\rad(t_c) = \rad_{max} \leq d_T$, from which we can conclude that $\Gamma(t_c) = \mathcal{B}(\eta, \rad_{max}) \subseteq \mathcal{B}(\eta, d_T) = \hat{\T} \subseteq \T$, proving the first condition.

(ii) For all $j \in [1;n_o]$, define the time-varying function,
\begin{equation}
    J^{(j)}(t)=(\cen(t)-o^{(j)}(t))^\top(\cen(t)-o^{(j)}(t))-(\rad^{(j)}_o(t)+\rad_{min})^2\nonumber
\end{equation}
which measures the squared distance between the STT center and the center of the $j^{th}$ unsafe set, offset by the safety margin $\rad^{(j)}_o(t)+\rad_{min}$.
The time derivative of $J^{(j)}(t)$ is
\begin{align}
    \dot{J}^{(j)}(t)&=2(\cen(t)-o^{(j)}(t))^\top\dot{\cen}(t)-2(\cen(t)-o^{(j)}(t))^\top\dot{o}^{(j)}(t) \notag \\
    & \qquad \qquad \qquad -2(\rad_o^{(j)}+\rad_{min})\dot\rad^{(j)}_o(t).
\end{align}
We analyze $\dot{J}^{(j)}(t)$ on the boundary of the safe margin around each obstacle, $\|\cen(t) - o^{(j)}(t)\| = \rad_o^{(j)}(t) + \rad_{min}.$
Substituting the expression for $\dot{\cen}(t)$ into $\dot{J}^{(j)}(t)$ yields:
\begin{align}\label{eqn:j_dot}
     &\dot{J}^{(j)}(t) = 2k_1\frac{t_c}{t_c-t}(\cen(t)-o^{(j)}(t))^\top (\eta-\cen(t)) \notag \\
     &+ 2k_{2,j}\Big(\frac{1}{\norm{\cen(t) - o^{(j)}(t)}-\rad_o^{(j)}(t)}-\frac{1}{\rad_{max}}\Big)\notag\\
     &\frac{\norm{\cen(t) - o^{(j)}(t)}^2}{\left( \| \cen(t) - o^{(j)}(t) \| - (\rad_o^{(j)}(t)  + \rad_{min})\right)^3} \notag \\
     &+ 2k_{3,j}\Big(\frac{1}{\norm{\cen(t) - o^{(j)}(t)}-\rad_o^{(j)}(t)}-\frac{1}{\rad_{max}}\Big) \notag\\
     &(\cen(t)-o^{(j)}(t))^\top m_j(t) - 2(\cen(t)-o^{(j)}(t)) \dot{o}^{(j)}(t)\\
     &-2(\rad_o^{(j)}+\rad_{min})\dot\rad^{(j)}(t). \notag
\end{align}
As $\| \cen(t) - o^{(j)}(t) \| \rightarrow \rad_o^{(j)}(t) + \rad_{min}$, the denominator in the second term approaches zero, making this term dominant and positive. As a result, $\dot{J}^{(j)}(t) > 0$ near the boundary.
The STT center is initially at a safe distance from the $j$-th unsafe set, i.e., $| \cen(0) - o^{(j)}(0) | > \rad_o^{(j)}(0) + \rad_{min}$, which implies $J^{(j)}(0) > 0$. Since $\dot{J}^{(j)}(t) > 0$ as $\| \cen(t) - o^{(j)}(t) \| \rightarrow \rad_o^{(j)}(t) + \rad_{min}$, the function $J^{(j)}(t)$ cannot decrease to zero or become negative. Therefore, $J^{(j)}(t) > 0$ holds for all $t \in [0, t_c]$, implying 
$$\| \cen(t) - o^{(j)}(t) \| > \rad_o^{(j)}(t) + \rad_{min} \ \text{ for all } \ t \in [0,t_c].$$
Hence, the STT center maintains a minimum separation of $\rad_{min}$ from the $j$-th obstacle at all times. 
Repeating this argument for all $j \in [1 ;n_o]$, we conclude that the STT center maintains a safe distance from each unsafe set for all time.

Now, to guarantee that the tube $\Gamma(t)=\mathcal{B}(\cen(t),\rad(t))$ does not intersect with any unsafe set, it suffices to show that:
\begin{align}\label{eqn:radprove}
    \rad(t)\leq \norm{\cen-o^{(j)}}-\rad_o^{(j)}(t),\forall j \in [1,..,n_o],
\end{align} 
We verify this using the solution of the radius dynamics given in Equation~\eqref{eqn:radius_dynamics}, with $d(t)$ defined in Equation~\eqref{eqn:rho}:
\begin{equation}\label{eqn:rad1}
    \rad(t) = -\frac{1}{\nu} \ln \left( \ex^{- \nu \rad_{max}} + \ex^{- \nu  d(t)} \right).
\end{equation} 
\\
Now, consider the following two cases:
\\
\textbf{Case 1:} $\forall j \in [1,n_o], \rad_{max} \leq \norm{\cen(t)-o^{(j)}(t)}-\rad_o^{(j)}(t)$:
\begin{align*}
     \rad(t) &= -\frac{1}{\nu} \ln \left( \ex^{- \nu \rad_{max}} + \ex^{- \nu  d(t)} \right) \\
     &\leq\min\Big(\min_{j={1,..,n_o}}\Big(\norm{\cen(t)-o^{(j)}(t)}-\rad_o^{(j)}\Big),\rad_{max}\Big) \\
     &\leq \rad_{max}.
\end{align*}
Therefore, for all $j \in [1,n_o]$, we have
$\rad(t) \leq \rad_{max} \leq \norm{\cen(t)-o^{(j)}(t)}-\rad_o^{(j)}(t),$
satisfying condition~\eqref{eqn:radprove}.
\\
\textbf{Case 2:} $\exists \hat j \in [1,n_o], \rad_{max} > \norm{\cen(t)-o^{(\hat j)}(t)}-\rad_o^{(\hat j)}(t)$:
\begin{align*}
     \rad(t) &= -\frac{1}{\nu} \ln \left( \ex^{- \nu \rad_{max}} + \ex^{- \nu  d(t)} \right) \\
     &\leq\min\Big(\Big(\norm{\cen(t)-o^{(\hat j)}(t)}-\rad_o^{(\hat j)}\Big),\rad_{max}\Big) \\
     &\leq \norm{\cen(t)-o^{(\hat j)}(t)}-\rad_o^{(\hat j)}(t),
\end{align*}
ensuring condition~\eqref{eqn:radprove} holds.

In both scenarios, \eqref{eqn:radprove} is satisfied, guaranteeing that the STT $\Gamma(t)$ does not intersect any unsafe set at any time:
$\Gamma(t) \cap \U(t) = \emptyset, \forall t \in [0, t_c].$

(iii)
From part (ii), we established that for all $j \in [1;n_o]$, the STT center $\cen(t)$ maintains a minimum safety distance of $\rad_o^{(j)}(t) + \rad_{min}$ from each unsafe set, i.e., for all $t \in [0,t_c]$,
$$\norm{\cen(t)-o^{(j)}(t)}\geq \rad_o^{(j)}(t) + \rad_{min} \implies  d(t) \geq \rad_{min},$$
Substituting this inequality into the closed-form expression for the radius in Equation~\eqref{eqn:rad1}, we obtain:
$$\rad(t) \geq -\frac{1}{\nu} \ln \left( \ex^{- \nu \rad_{max}} + \ex^{- \nu \rad_{min}} \right) > 0, \quad \forall t \in [0,t_c].$$
Thus, the STT radius remains strictly positive at all times.

Hence, the STT $\Gamma(t)$ reaches the target $\T$ at the prescribed time $t_c$, avoids the unsafe set $\U(t)$, and maintains a guaranteed positive radius throughout. This completes the proof that the proposed STT satisfies the T-RAS specification.
\end{proof}

\begin{lemma}\label{lemma:cont}
    The STT center $\cen(t)$, radius $\rad(t)$ and their time derivatives $\dot{\cen}(t)$, $\dot{\rad}(t)$ are all continuous and uniformly bounded for all $t \in [0,t_c]$.
\end{lemma}

\begin{proof}
From the radius dynamics \eqref{eqn:radius_dynamics} and the use of a smooth approximation of the $\min$ function, both $\rad(t)$ and $\dot{\rad}(t)$ are continuous and bounded for all $t \in [0,t_c]$. Moreover, since $\dot{J}^{(j)} > 0$ near the unsafe boundary, the STT center is repelled from the unsafe set, ensuring that $\| \cen(t) - o^{(j)}(t) \|$ stays strictly greater than $\rad_o^{(j)} + \rad_{min}$ at all times. This implies that $\cen(t)$ and $\dot{\cen}(t)$, which depend smoothly on unsafe sets and the target, are also continuous and bounded over the time horizon.
\end{proof}

\begin{remark}
    In Equation~\eqref{eqn:cen_dynamic}, the first term approaches zero as $t \rightarrow t_c$, provided that $k_1 t_c > 1$, ensuring the continuity of $\dot{\cen}(t)$. This results in a smooth deceleration as the robot approaches the target, avoiding chatter near the goal. The condition $k_1 t_c > 1$ is not restrictive in practice, as a very small $t_c$ would require an unrealistically fast motion, which is typically impractical for physical systems.
\end{remark}
{
\begin{remark}
    The proposed real-time STT framework adapts the tube dynamically based on local, real-time obstacle information. As it does not perform a global search, the method is sound but not complete, i.e., it can guarantee safe task satisfaction when a feasible STT is found, but does not ensure that such a tube can always be found.
\end{remark}
}

\section{Controller Design}
We now derive a closed-form, {model-free} control law to constrain the system output within the STT \eqref{eqn:stt_con}. {It is important to note that, unlike traditional STT formulations \cite{das2024spatiotemporal} that used hyper-rectangular tubes, this work introduces a spherical (ball-shaped) STT. This modification requires a distinct control law and Lyapunov-based stability analysis, although the overall backstepping-like structure \cite{PPCfeedback} of the derivation remains similar.} We first design an intermediate input $r_2$ to enforce the tube constraint on the output, and then recursively define intermediate signals $r_{k+1}$ so that each state $x_k$ tracks its reference $r_k$, with $u = r_{N+1}$ as the control input.

The steps of the control design are as follows.

\textbf{Stage $1$}: Given the STT $\Gamma(t)$ in \eqref{eqn:stt_ball}, define the normalized error $e_1(x_1,t)$ and the transformed error $\varepsilon_1(x_1,t)$ as
\begin{align}
e_1(x_1,t) = \frac{\norm{x_1(t) - \cen(t)}}{\rad(t)},
\varepsilon_1(x_1,t) = \ln\left( \frac{1 + e_1(x_1,t)}{1 - e_1(x_1,t)} \right).\nonumber
\end{align}
The intermediate control input $r_2(x_1,t)$ is then given by 
\begin{equation}
    r_2(x_1,t) = -\kappa_1 \varepsilon_1(x_1,t) \big( x_1(t)-\cen(t) \big), \kappa_1 \in \R^+.
\end{equation}

\textbf{Stage $k$} ($k \in [2;N]$): 
To ensure $x_k$ tracks the reference signal $r_k$ from Stage $k-1$, we define a time-varying bound: $\gamma_{k,i}(t) = (p_{k,i} - q_{k,i})\ex^{-\mu_{k,i}t} + q_{k,i}$, and enforce,
    $-\gamma_{k,i}(t) \leq (x_{k,i}-r_{k,i}) \leq \gamma_{k,i}(t), \forall (t,i) \in \R_0^+ \times [1;n],$
where, $\mu_{k,i} \in \R_0^+$, and $p_{k,i}> q_{k,i} \in \R^+$ are chosen such that $|x_{k,i}(0) - r_{k,i}(0)| \leq p_{k,i}$.
Now, define the normalized error $e_k(x_{k},t)$, the transformed error $\varepsilon_k(x_{k},t)$ and $\xi_k(x_{k},t)$ as follows
\begin{subequations} \label{eq:stage k}
    \begin{align}
    e_k(x_{k},t) &= [e_{k,1}(x_{k,1},t), \ldots, e_{k,n}(x_{k,n},t)]^\top \\
    &= (\textsf{diag}(\gamma_{k,1}(t),\ldots,\gamma_{k,n}(t)))^{-1} \left(x_{k} - r_k \right), \notag \\
    \varepsilon_k(x_{k},t) &= \big[\ln\left(\frac{1+e_{k,1}(x_{k,1},t)}{1-e_{k,1}(x_{k,1},t)}\right), \ldots, \notag \\
    &\hspace{2em}\ln\left(\frac{1+e_{k,n}(x_{k,n},t)}{1-e_{k,n}(x_{k,n},t)}\right) \big]^\top, \\
        \xi_k(x_{k},t) &= 4 \big(\textsf{diag}(\gamma_{k,1}(t),\ldots,\gamma_{k,n}(t)) \big)^{-1} \notag \\
        &\hspace{5em}(I_n-\textsf{diag}(e_k \circ e_k))^{-1}.
\end{align}
\end{subequations}
The next intermediate control input $r_{k+1}(\overline{x}_{k},t)$ is then:
\begin{equation*}
    r_{k+1}(\overline{x}_{k},t) = - \kappa_k\xi_k(x_{k},t)\varepsilon_k(x_{k},t), \kappa_k \in \R^+.
\end{equation*}
At stage $N$, this intermediate input is the actual control input:
\begin{equation*}
    u(\overline{x}_{N},t) = - \kappa_N\xi_N(x_{N},t)\varepsilon_N(x_{N},t), \kappa_N \in \R^+.
\end{equation*}

We now state the main theorem guaranteeing that this controller enforces the desired T-RAS behavior.

\begin{figure*}[t]
    \centering
    \includegraphics[width=0.90\textwidth]{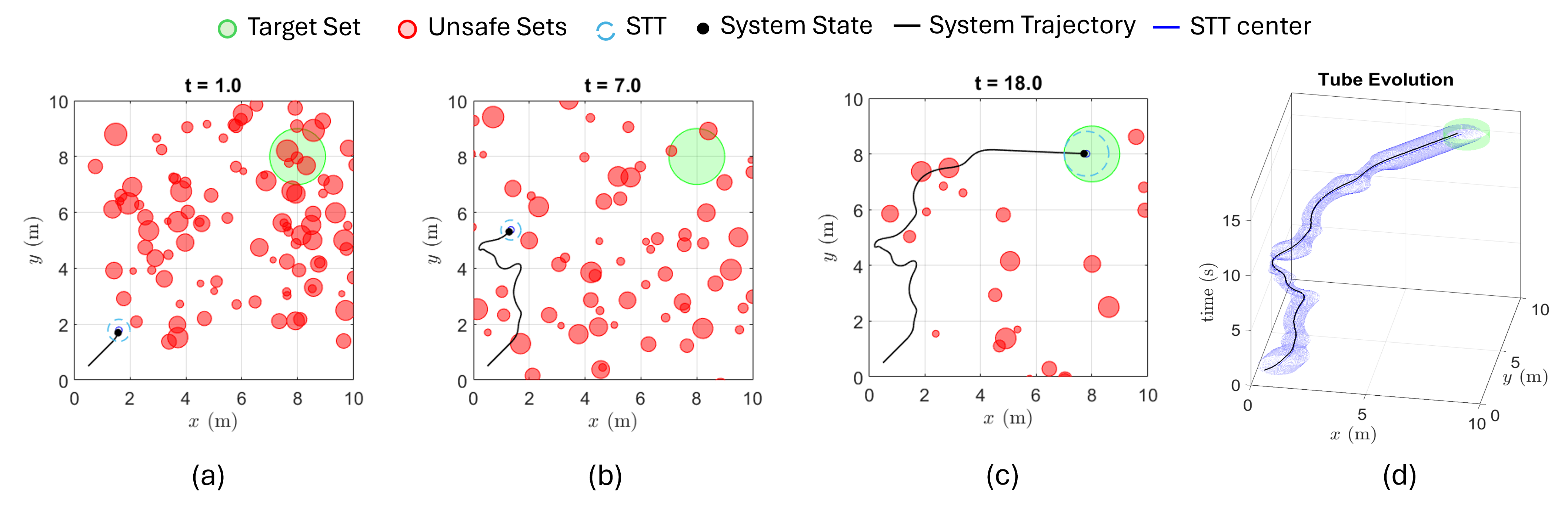}
    \caption{Omnidirectional Mobile Robot. (a)--(c) System trajectory at different time instants. (d) STT evolution. \href{https://youtu.be/BjgmxghPdWk}{[Video]}}
    \label{fig:omni_sim}
\end{figure*}

\begin{theorem} \label{theorem_ras}
    Consider the nonlinear MIMO system \eqref{eqn:sysdyn} under Assumptions \ref{assum:lip} - \ref{assum:pd}, with a T-RAS task (Definition~\ref{def:tras}), and the STT $\Gamma(t)$ from \eqref{eqn:stt_ball}.
    If the initial output is within the STT: $y(0) \in \Gamma(0)$, 
    then the controller
    \begin{align}\label{eqn:Control_ras}
        r_2(x_1,t) &= -\kappa_1 \varepsilon_1(x_1,t) \left( x_1(t)-\cen(t) \right), \kappa_1 \in \R^+, \notag \\
        r_{k+1}(\overline{x}_{k},t) &= - \kappa_k\xi_k(x_{k},t)\varepsilon_k(x_{k},t), k \in [2;N-1], \notag \\
        u(\overline{x}_{N},t) &= - \kappa_N\xi_N(x_{N},t)\varepsilon_N(x_{N},t),
    \end{align}    
    ensure that the system output remains within the STT: 
    $$y(t) = x_1(t) \in \Gamma(t), \forall t \in [0, t_c],$$
    thereby satisfying the desired T-RAS specification.
\end{theorem}

\begin{proof}
We prove the correctness of the proposed control law for Stage $1$, and for Stages $2$ through $N$, we refer to \cite{das2024spatiotemporal}, as the arguments follow identically.

\textbf{Stage $1$:}  
Differentiating $e_1(x_1,t)$ with respect to time $t$ and substituting the system dynamics from \eqref{eqn:sysdyn}, we obtain:
\begin{align}
    \dot{e}_1(x_1,t) &= \Big( \|x_1 -\cen\|^{-1}(x_1 - \cen)^\top(f_1(x_1) + g_1(x_1)x_2 \notag \\
    &- \dot{\cen}(t)) - \dot{\rad}(t)e_1(x_1,t) \Big) / \rad(t) := h_1(e_1,t).
\end{align}
We define the error constraints for $e_1$ through the open and bounded set $\mathbb{D}:=(0,1)$.
The proof proceeds in three steps. First, we show that a maximal solution exists within $\mathbb{D}$ in the maximal time solution interval $[0, \tau_{\max})$. Next, we prove that the proposed control law \eqref{eqn:Control_ras} ensures $e_1(x_1,t)$ remains in a compact subset of $\mathbb{D}$. Finally, we prove that $\tau_{\max}$ can be extended to $\infty$.

\underline{\textit{Step (i):}}  
Given $\|x_1(0) - \cen(0)\| \leq r(0)$, the initial error $e_1(x_1(0),0)$ lies in $\mathbb{D}$. Since $\cen(t)$, $\rad(t)$ are smooth and bounded (Lemma~\ref{lemma:cont}), $f_1(x_1)$, $g_1(x_1)$ are locally Lipschitz, and the control law $r_2(x_1,t)$ is smooth in $\mathbb{D}$, the dynamics $h_1(e_1,t)$ is locally Lipschitz in $e_1$ and continuous in $t$. Hence, by \cite[Theorem 54]{sontag}, there exists a maximal solution $e_1 : [0, \tau_{\max}) \rightarrow \mathbb{D}$ such that $e_1(t) \in \mathbb{D}$ for all $t \in [0, \tau_{\max})$.

\underline{\textit{Step (ii):}}  
Consider the Lyapunov candidate $V_1 = 0.5\varepsilon_1^2$. Differentiating $V_1$ w.r.t. $t$, and substituting $\dot{\varepsilon}_1$, $\dot{e}_1$, and the system dynamics with the control law \eqref{eqn:Control_ras}, we get:
\begin{align*}
    &\dot{V}_1 = \varepsilon_1 \dot{\varepsilon}_1
    =\frac{2\varepsilon_1}{\rad(1-e_1^2)} \Big( \frac{(x_1-\cen)^\top}{\norm{x_1-\cen}}(\dot{x}_1-\dot{\cen})-\dot{\rad}e_1 \Big) \\
    &=  \frac{2\varepsilon_1}{\rad(1-e_1^2)} \Big( \frac{(x_1-\cen)^\top}{\norm{x_1-\cen}}g_1(x_1)x_2+ {\Phi_1}\Big)\\
    &=\frac{2}{\rad(1-e_1^2)}\Big(\frac{-\varepsilon_1 \kappa}{\norm{x_1-\cen}}(x_1-\cen)^\top g_1(x_1)(x_1-\cen)+\varepsilon_1 \Phi_1 \Big),  
\end{align*}
where $ \Phi_1 := \frac{(x_1 - \cen)^\top}{\|x_1 - \cen\|}(f_1(x_1) + d_1 - \dot{\cen}(t) - \dot{\rad}(t)e_1).$
Using the Rayleigh-Ritz inequality and Assumption~\ref{assum:pd}, we have,
$\underline{g}_1\norm{x_1-\cen}^2 \leq \lambda_{min}(g_1(x_1))\norm{x_1-\cen^2}^2 \leq (x_1-\cen)^\top g_1(x_1)(x_1-\cen)$,
which leads to:
\begin{align*}
\dot{V}_1 &\leq \alpha_1 \left(-\kappa \varepsilon_1^2 \underline{g}_1 \|x_1 - \cen\|^2 + \varepsilon_1 \|\Phi_1\| \right), 
\end{align*}
where $\alpha_1 = \frac{2}{\rad(1 - e_1^2)} > 0$.
From Lemma~\ref{lemma:cont}, the functions $\cen(t)$, $\dot{\cen}(t)$, $\rad(t)$, $\dot{\rad}(t)$ are all bounded. Since $x_1(t)$ remains in the tube by Step (i), and $f_1$, $g_1$ are continuous, it follows that $\|\Phi_1\| < \infty$ for all $t \in [0, \tau_{\max})$.
Now, for some $\theta \in (0,1)$, we add and subtract $\kappa \alpha_1 \varepsilon_1^2 \underline{g}_1 \theta \|x_1 - c\|^2$:
\begin{align*}
       \dot{V}_1& \leq \alpha_1 \Big (-\kappa_1 \varepsilon_1^2 \underline{g}_1 (1-\theta)\norm{x_1-\cen}^2\\
       & \hspace{1.8cm} -(\kappa_1 \underline{g}_1 \varepsilon_1^2 \theta \norm{x_1-\cen}^2- \|\varepsilon_1\| \norm{\Phi_1}) \Big)\nonumber\\
       &\leq -\alpha_1 \varepsilon_1^2 \underline{g}_1 (1-\theta) \norm{x_1-\cen}^2, \\
       &\hspace{2cm} \forall \kappa_1 \underline{g}_1 \|\varepsilon_1\| \theta \norm{x_1-\cen}^2-\norm{\Phi_1} \geq 0\nonumber \\
       &\leq  -\alpha_1 \varepsilon_1^2 \underline{g}_1 (1-\theta) \norm{x_1-\cen}^2, \\
       &\forall \norm{\varepsilon_1}\geq \frac{\norm{\Phi_1}}{\kappa_1 \underline{g}_1 \theta \norm{x_1-\cen}^2} := \varepsilon_1^*, \ \forall t\in [0.\tau_{max}).
\end{align*}

Thus, we can conclude that there exists a time-independent upper bound $\varepsilon_1^* \in \mathbb{R}^+_0$ to the transformed error, i.e., $\|\varepsilon_1\| \leq \varepsilon_1^*$, for all $t\in[0,\tau_{max})$. Inverting the transformation, we can express the bounds on $e_1$ as:
\begin{align*}
    0 \leq e_{1} \leq \overline{e}_{1} := \frac{e_{1}^{\varepsilon_{1}^*}-1}{e_{1}^{\varepsilon_{1}^*}+1} < 1.
\end{align*}
Thus, $e_1(t) \in [0, \overline{e}_{1}] =: \mathbb{D}' \subset \mathbb{D}$ for all $t \in [0, \tau_{\max})$.

\underline{\textit{Step (iii):}}  
Since $e_1(t)$ remains in the compact subset $\mathbb{D}' \subset \mathbb{D}$ for all $t \in [0, \tau_{\max})$, the solution cannot escape $\mathbb{D}$ in finite time. By contradiction, if $\tau_{\max} < \infty$, then by \cite[Prop. C.3.6]{sontag}, there exist $t' < \tau_{\max}$ such that $e_1(t') \notin \mathbb{D}$, which contradicts Step (ii). Hence, the solution exists for all $t \geq 0$, i.e., $\tau_{\max} = \infty$.

\textbf{Stages $k \in [2, N]$:}  
For the remaining stages, we apply the same reasoning as presented in Theorem 4.1 of \cite{das2024spatiotemporal}, and is thus omitted here for brevity.

This completes the proof, showing that the control law in \eqref{eqn:Control_ras} enforces the tube constraint \eqref{eqn:stt_con}, and thereby ensures satisfaction of the T-RAS task.
\end{proof}
{
\begin{remark}
The time-dependent control law in \eqref{eqn:Control_ras} is closed-form and model-free, ensuring satisfaction of the T-RAS specifications for control-affine MIMO pure-feedback systems, requiring no explicit knowledge or approximation of the system dynamics $f$ and $g$.
\end{remark}
}

\begin{figure}[t]
    \centering
    \includegraphics[width=0.7\textwidth]{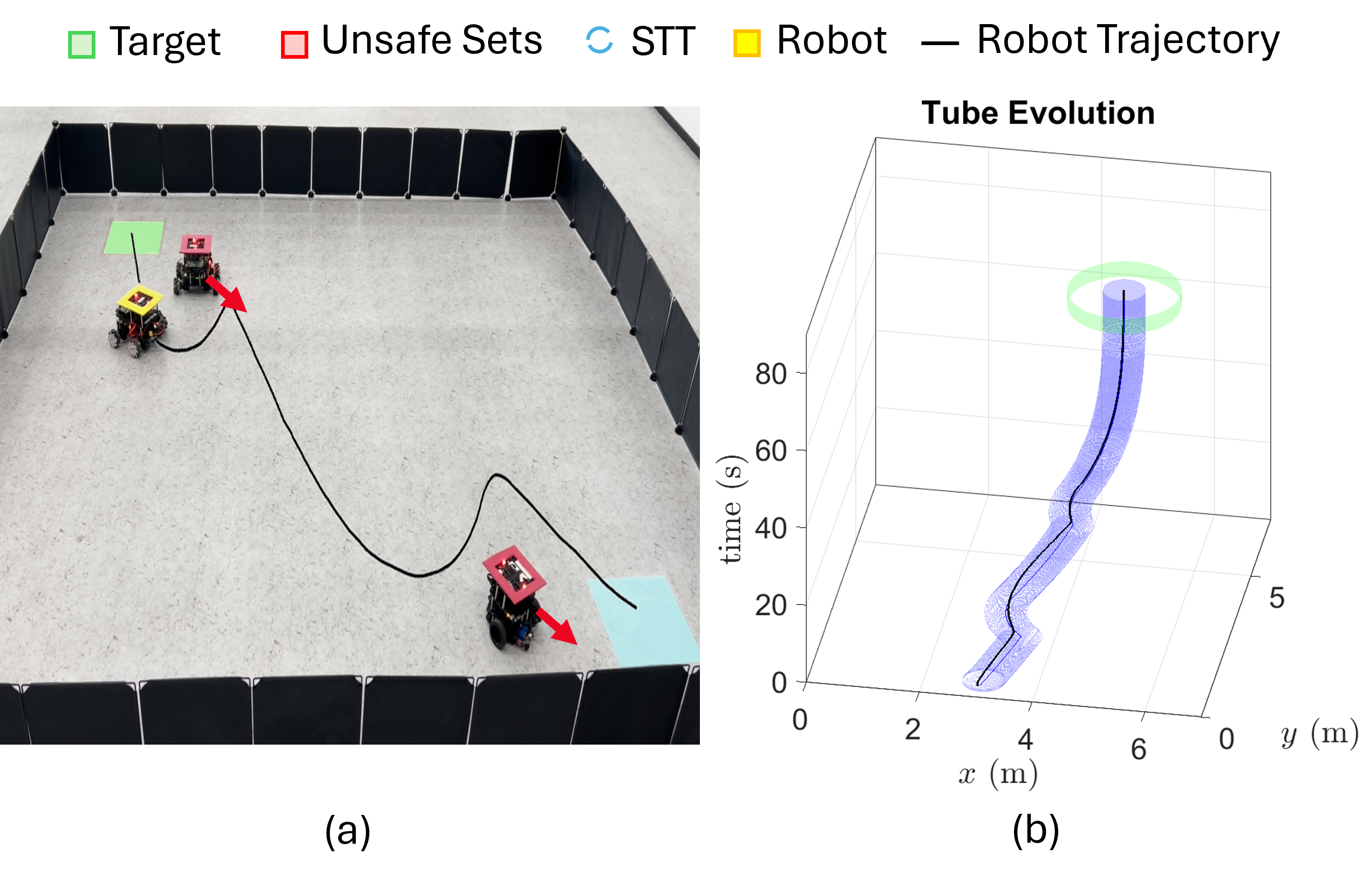}
    \caption{Mobile Robot Hardware Results. (a) Robot trajectory. (b) STT evolution. \href{https://youtu.be/BjgmxghPdWk}{[Video]}}
    \label{fig:omni_hardware}
\end{figure}

\section{Case Studies}
We validate the proposed real-time STT framework through two case studies, a 2D mobile robot and a 3D quadrotor, supported by simulations and real-world robot experiments.
{The parameters used for these examples are listed in Table~\ref{tab:params}.}

\subsection{Mobile Robot}
We consider an omnidirectional mobile robot in a 2D environment with dynamics adapted from \cite{NAHS}. To test the robustness of our approach, we also introduce unknown but bounded disturbances. The robot starts from the initial set $\So = \mathcal{B}([1,1]^\top,1)$ and must reach the target $\T = \mathcal{B}([8,8]^\top,1)$ within a prescribed time $t_c = 18$ s, while avoiding multiple dynamic obstacles.
Figure~\ref{fig:omni_sim} shows the robot’s trajectory at different time stamps and the STT. 

We also experimentally validated the approach using a hardware setup with a robot navigating around two moving obstacles. Figure~\ref{fig:omni_hardware} shows the setup and the STT. The simulation and hardware videos are available at \href{https://youtu.be/BjgmxghPdWk}{Video Link}.

\begin{figure}[t]
    \centering
    \includegraphics[width=0.7\textwidth]{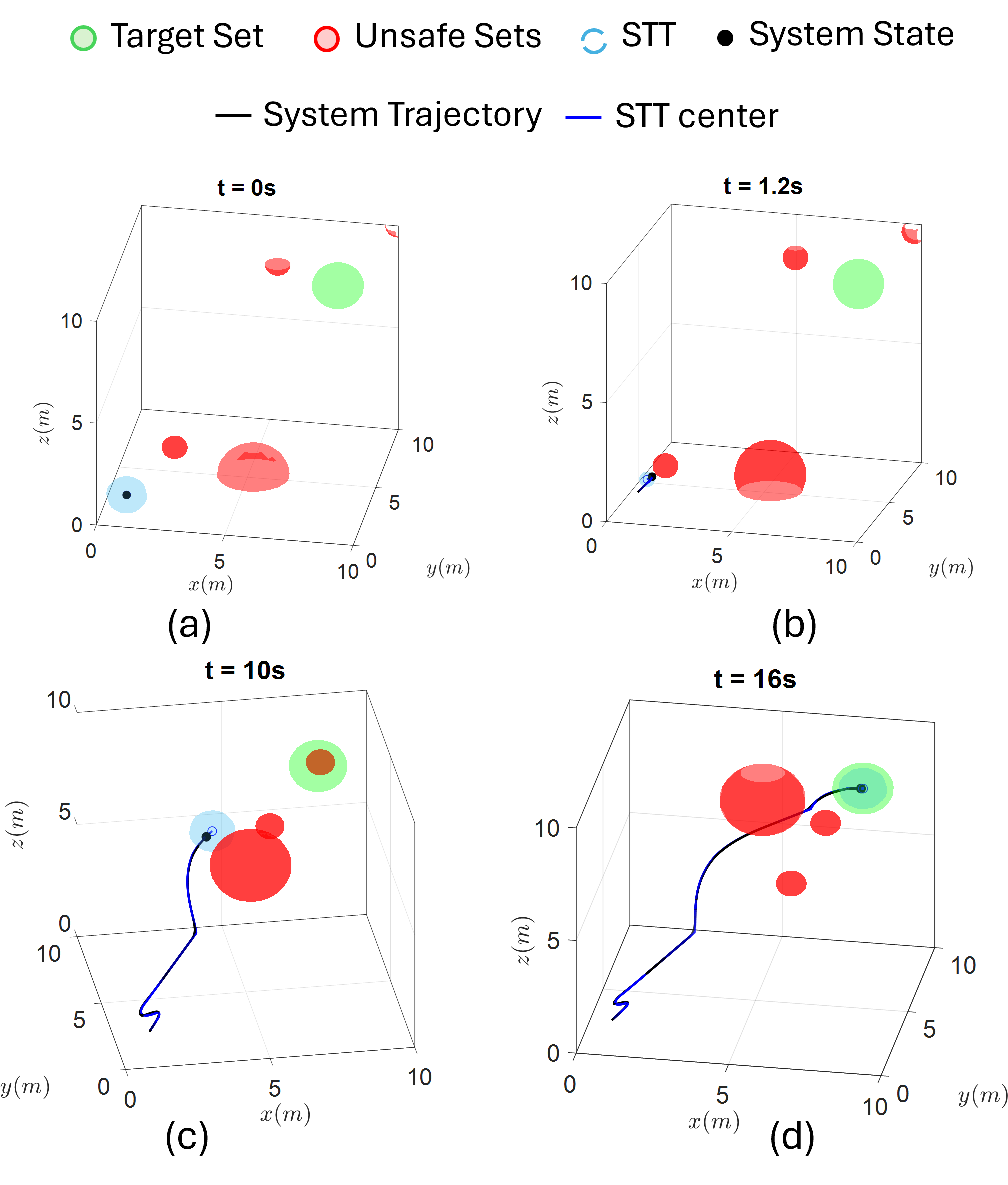}
    \caption{Quadrotor Simulation. System trajectory at different time instants. \href{https://youtu.be/BjgmxghPdWk}{[Video]}}
    \label{fig:uav_sim}
\end{figure}

\subsection{Quadrotor}
We further evaluate the framework on a Quadrotor operating in a 3D environment with second-order dynamics adapted from \cite{APF_drone}. As STTs explicitly address robustness to disturbances, we include unknown but bounded disturbance terms.
The system starts from $\So = \mathcal{B}([1,1,1]^\top, 1)$ and must reach the target $\T = \mathcal{B}([8,8,8]^\top, 1)$ within a prescribed time of $t_c = 16$~s, while avoiding multiple dynamic obstacles.
Figure~\ref{fig:uav_sim} shows the Quadrotor's trajectory and the surrounding STT at different times. 
The simulation video is available at \href{https://youtu.be/BjgmxghPdWk}{Video Link}.

\begin{table}[t]
\caption{Experimental Parameters}
\label{tab:params}
\centering
\renewcommand{\arraystretch}{1.2}
\begin{tabular}{lcc}
\hline
\textbf{Parameter} & \textbf{Mobile Robot} & \textbf{Quadrotor} \\
\hline
Initial center $\cen(0)$ & $[1,1]^\top$ & $[1,1,1]^\top$ \\
Initial radius $\rad(0)$ & $0.9$ & $0.9$ \\
Maximum radius $\rad_{\max}$ & $0.9$ & $0.9$ \\
Minimum radius $\rad_{\min}$ & $0.1$ & $0.1$ \\
Smoothening parameter $\nu$ & $8$ & $8$ \\
Gains $(k_1,k_{2,j},k_{3,j})$ & $(600,600,600)$ & $(300,1000,300)$ \\
\hline
\end{tabular}
\end{table}

\begin{table*}[t]
\caption{Comparison With Baseline Algorithms}
\label{tab:comparison}
\centering
\renewcommand{\arraystretch}{1.2}
\begin{tabular}{lcccccccc}
\hline
\multirow{2}{*}{\textbf{Method}} &
\multicolumn{2}{c}{\textbf{Computation Time (ms)}} &
\multicolumn{2}{c}{\textbf{Path Length (m)}} &
\multicolumn{2}{c}{\textbf{Path Smoothness}} &
\multicolumn{2}{c}{\textbf{Success Rate (\%)}} \\
\cline{2-9}
& Mean & SD & Mean & SD & Mean & SD & Nominal & Disturbed \\
\hline
DWA & 23.113 & 14.111 & 11.177 & 1.626 & 2.991 & 7.522 & 96.0 & 94.0 \\
APF & 0.023 & 0.012 & 10.401 & 0.765 & 0.376 & 0.317 & 66.0 & 58.0 \\
NMPC & 19.070 & 9.213 & 10.992 & 1.075 & 1.055 & 1.771 & 98.0 & 98.0 \\
CBF & 2.322 & 0.962 & 10.373 & 0.741 & 1.566 & 4.155 & 100.0 & 86.0 \\
\textbf{Real-time STT} & \textbf{0.037} & \textbf{0.017} & \textbf{11.540} & \textbf{1.440} & \textbf{1.570} & \textbf{3.930} & \textbf{100.0} & \textbf{100.0} \\
\hline
\end{tabular}
\end{table*}

\begin{table}[t]
\caption{Scalability and Safety Results}
\label{tab:comparison_2}
\centering
\renewcommand{\arraystretch}{1.2}
\begin{tabular}{ccccc}
\hline
\multirow{2}{*}{$n_o$} &
\multicolumn{2}{c}{\textbf{2D Mobile Robot}} &
\multicolumn{2}{c}{\textbf{3D Quadrotor}} \\
\cline{2-5}
& Min. Clear. (m) & Time (s) & Min. Clear. (m) & Time (s) \\
\hline
1 & 0.0760 & 0.0051 & 0.4830 & 0.0055 \\
10 & 0.1657 & 0.0192 & 0.2489 & 0.0221 \\
50 & 0.0993 & 0.1106 & 0.1640 & 0.0895 \\
100 & 0.0464 & 0.2460 & 0.2528 & 0.1929 \\
\hline
\end{tabular}
\end{table}

\begin{figure}[t]
    \centering
    \includegraphics[width=0.5\textwidth]{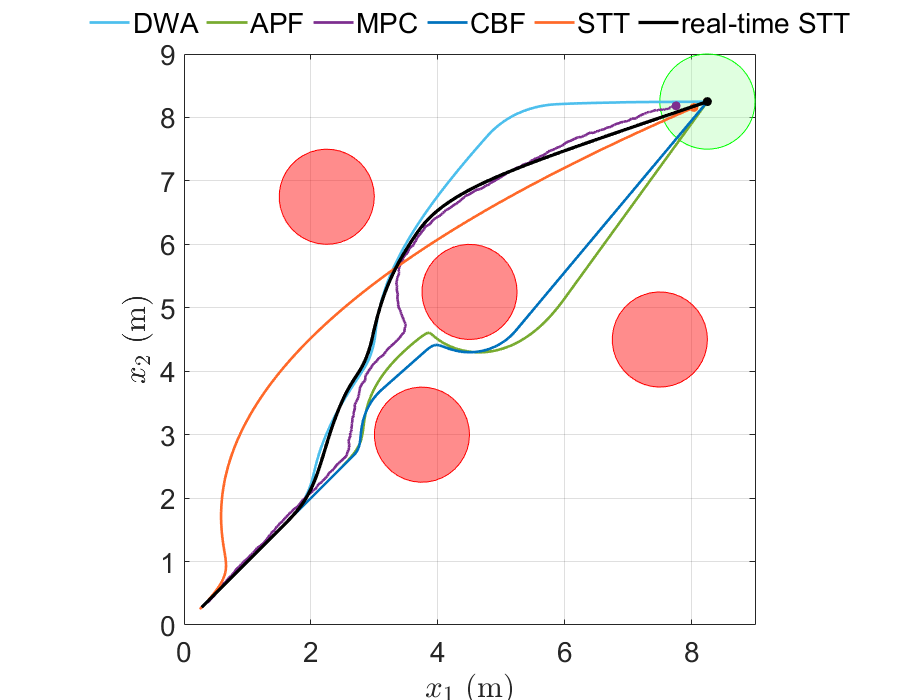}
    \caption{Trajectory Comparison Between Baseline Methods and the Proposed Real-Time STT Approach}
    \label{fig:compare}
\end{figure}

{
\section{Comparison and Discussion}
We compare the proposed real-time STT controller with four baseline algorithms: DWA \cite{fox2002dynamic}, APF \cite{APF_global}, NMPC \cite{MPC}, and CBF \cite{CBF-RRT1}. The evaluation is done over 100 randomized dynamic environments with 5–50 moving obstacles. Each method is tested in terms of computation time, path length, path smoothness, and success rate (percentage of runs completed without collision), both with and without bounded disturbances.
}

{
As shown in Table~\ref{tab:comparison}, the proposed STT framework achieves the highest success rate and the lowest computation time, even under disturbances. DWA and NMPC perform reasonably well but with significantly higher computation times, while APF struggles with local minima as the number of obstacles increases. CBF performs well in nominal conditions but degrades when system dynamics are uncertain. Moreover, none of these baseline methods provides prescribed-time guarantees.}

{
Table~\ref{tab:comparison} also reports the path length and smoothness (quantified as the mean curvature of the trajectory and a lower mean curvature indicates a smoother and more consistent path). The proposed framework generates smooth and collision-free trajectories, though not necessarily the shortest paths. This is because the controller operates based only on local, real-time information rather than global knowledge. In future work, we plan to combine the real-time STT controller with a global planner to achieve shorter, more optimal paths.
}

{
Figure~\ref{fig:compare} shows trajectories for a static environment. While the STT-based control law in \cite{das2024spatiotemporal} is closed-form and fast during execution, its offline tube generation step is computationally expensive, taking over 5.6~s even for this simple static case. Although the offline STT synthesis can handle dynamic obstacles if their trajectories are known in advance, its computational cost grows rapidly with the number of obstacles, making it unsuitable for real-time deployment in dense, unpredictable environments. 
}

{
To highlight the scalability of the proposed framework, we further evaluated how performance changes with increasing obstacle density and system dimensionality. Specifically, we tested setups with 1, 10, 50, and 100 obstacles in both 2D and 3D environments, running 50 randomized trials for each configuration. For each case, we measured the minimum clearance from obstacles and the average computation time. The results in Table~\ref{tab:comparison_2} show that the proposed approach consistently maintains safe separation and low computation times in all trials, demonstrating that the real-time STT framework remains robust and scalable even in dense, high-dimensional environments.}

\section{Conclusion}
In this work, we presented a real-time framework for synthesizing STTs that ensure safe control of nonlinear pure-feedback systems with unknown dynamics, under T-RAS tasks. The proposed approach dynamically adapts the tube’s center and radius using only real-time information, allowing safe operation in dynamic and partially unknown environments. We provided formal guarantees for safety and prescribed-time convergence and demonstrated the approach’s effectiveness and scalability through simulations and hardware experiments with a 2D mobile robot and a 3D UAV in cluttered, dynamic settings.

In future work, we plan to extend this framework to handle multi-agent coordination, integrate high-level temporal logic specifications, and {address systems with non-smooth or hybrid dynamics, and explicit input constraints. We also aim to combine the proposed controller with global planning to improve task feasibility in complex environments.}

\bibliographystyle{unsrt} 
\bibliography{sources} 

\end{document}